\documentclass{svproc}

\usepackage{url}

\usepackage{graphicx}
\usepackage{hyperref}
\usepackage{amsmath}
\usepackage{amsfonts}
\usepackage{amssymb}
\usepackage{xspace}

\def\define#1{{\em #1}}
\def\ie{i.~e.~}
\def\eg{e.~g.~}

\def\abs#1{\left| #1 \right|}
\def\mset#1{\left\{ #1 \right\}}

\def\taxis{\vartheta}
\def\axis{q}

\def\rotx#1{\hbox{R}_x({#1})}

\def\rotz#1{\hbox{R}_z({#1})}
\def\transx#1{\hbox{T}_x({#1})}

\def\transz#1{\hbox{T}_z({#1})}
\def\world{{\sf W}}

\def\transx#1{\hbox{T}_\pvar({#1})}

\def\transz#1{\hbox{T}_z({#1})}

\def\hvec#1#2{{}^{{#1}}#2}
\def\hframe#1#2{{}^{{#1}}_{{#2}}T}

\def\mathword#1{\mathchoice{\hbox{\sf #1}}
					  {\hbox{\sf #1}}
					  {\hbox{\sf\tiny #1}}
					  {\hbox{\sf\tiny #1}}
}
\def\tcp{\ifmmode \mathword{TCP}{} \else \text{TCP}\xspace\fi}
\def\TCP{\ifmmode \mathword{TCP}{} \else \text{TCP}\xspace\fi}
\def\world{{\mathword{W}}}

\def\backward{{\sf backward}}
\def\forward{{\sf forward}}
\def\world{{\sf W}}

\def\kw#1{{\texttt{#1}}}

\def\wcp{\ifmmode \mbox{\sf WCP} \else \mbox{\sf WCP}\xspace\fi}

\def\backward{\ifmmode \mbox{\sf backward} \else \mbox{\sf backward}\xspace\fi}
\def\forward{\ifmmode \mbox{\sf forward} \else \mbox{\sf forward}\xspace\fi}
\def\status{{s}}
\def\Status{{\cal S}}
\def\status{{c}}
\def\Status{{\cal C}}
\def\fStatus{C}

\def\mstar{m^\star}
\def\sstar{{\status^\star}}
\def\statusstar{\sstar}
\def\wstar{w^\star}
\def\xstar{\pvar^\star}

\def\unionset{\cup}

\def\axis{q}
\def\frames{{\cal F}}
\def\R{{\mathbb{R}}}

\def\refe#1{(\ref{#1})}
\def\st{\text{subject to }}

\begin{document}
\mainmatter              
\title{Optimization of Cartesian Tasks with Configuration Selection}
\titlerunning{Optimization with Configuration Selection}  
%
\author{Martin G.~Wei\ss}
\authorrunning{Martin G.~Wei\ss} 
%
\tocauthor{Martin G.~Wei\ss}
\institute{Ostbayerische Technische Hochschule Regensburg, Germany\\
\email{martin.weiss@oth-regensburg.de}}

\maketitle              

\begin{abstract}
A basic task in the design of an industrial robot application is the relative placement
of robot and workpiece. Process points are defined in Cartesian coordinates relative
to the workpiece coordinate system, and the workpiece has to be located such that the robot
can reach all points. Finding such a location is still an iterative procedure based on the 
developers' intuition. One difficulty is the choice of one of the several solutions of the 
backward transform of a typical 6R robot. 
We present a novel algorithm that simultaneously optimizes the workpiece location and the 
robot configuration
at all process points using higher order optimization algorithms. A key ingredient is 
the extension of the robot with a virtual prismatic axis.
The practical feasibility of the approach is shown with an example 
using a commercial industrial robot.

\end{abstract}

\def\pvar{{x}}
\def\SE#1{\mbox{SE}( #1 )}

\section{Problem Statement}

When programming an industrial robot application the typical workflow starts with the workpiece that has
to be processed. The process points like welding points, drilling holes, points describing a glueing 
or laser contour are defined. For handling applications also points are defined, not relative to 
the transported workpiece, but to devices holding and 
transporting the items. The points $P^k$ are actually frames in $\SE 3$
relative to the workpiece frame $F$. 
Then the application developer chooses a robot, and the workpiece location
relative to the robot.

The difficulty is as follows: A typical industrial robot 
with up to 8 discrete solutions, one has to be chosen individually for each process point according to some criterion, solvability being the first. 
We call these solutions in axis space \define{configurations} $\status\in \Status$, with $\Status$ a finite set coding the possible configurations
like $\Status = \mset{0, 1, \ldots, 7}$.  
In industrial robot programming languages frame data are
extended by a code for the configuration selection. E.g. the KRL language 
\cite{KSS83SI} uses an integer named 
status, abbreviated \kw{S}, whose bits 0,1 and 2 code sign choices in the backward transform. 
A point-to-point command then looks like 
\begin{verbatim}
       PTP {X 100, Y 200, Z 300, A 40, B 50, C 60, S 'B101'}
\end{verbatim}
with components \kw{A}, \kw{B}, \kw{C} as Euler-like angles $\alpha, \beta, \gamma$. Other robot languages use different terms and syntax but the underlying mathematics is the same. 

The configuration information is not a part of the Cartesian data so actually an additional degree of freedom 
comes from the backward transform. 
But it is not intuitive for humans which solution falls into axis limits due to mechanical design, which
may have to be restricted further according to the cell setup or cabling of the robot. 
So not all of the 8 solutions may reachable, even in an asymmetric way
as axis ranges are not symmetric around $0^o$. 
So making all points reachable for the robot is still based on human intuition and trial-and-error.
It would be desirable to have an algorithm which, given  process frames 

${P^1}, \ldots,  {P^K}$ 
expressed in a frame $F$, 
 and a robot description, determines a reachable workpiece position $F$ in world $\world$
{\em and} suitable configurations for all $P^k$.

In \cite{Wei.2019} we have presented an algorithm which solves the task described so far with an 
extension of the 6R robot to a RRRPRRR robot with a virtual prismatic axis that 
makes all of $\SE 3$ reachable, if we drop axis restrictions in addition. The virtual axis measures 
non-reachability from the original robot perspective. If a location $F$ is found with the
virtual axis set to 0 for all positions, and all axes are inside their ranges, then 
the problem is also solved for original robot. However this algorithm only works 
with the restrictive assumption that all $P^k$ are approached with the same configuration $\status\in\Status$.  
In this paper we extend the approach to different configurations $\status_k$ for each $P^k$, still 
determined by efficient algorithms from differentiable optimization.
Approaches with virtual axes are also used e.g. in \cite{GeuFlores.2019,Leger.2016,Pellegrino.2020} but 
the configuration is explicitly held fixed in these papers. Automatic 
selection of the robot configuration is new to the best of the author's knowledge. 

The paper is organized as follows: In Section 2 we introduce notation and 
briefly present the virtual axis approach. 
Section 3 explains the formulation of the optimization problem with the discrete 
configuration set suitable for solvers using derivatives. 
Numerical results with data from a industrial 
robot are shown in Section 4 before we summarize and conclude with directions for further research.

\section{Virtual Axis Approach}

\begin{figure}
\hfill
\includegraphics[width=0.55\textwidth]{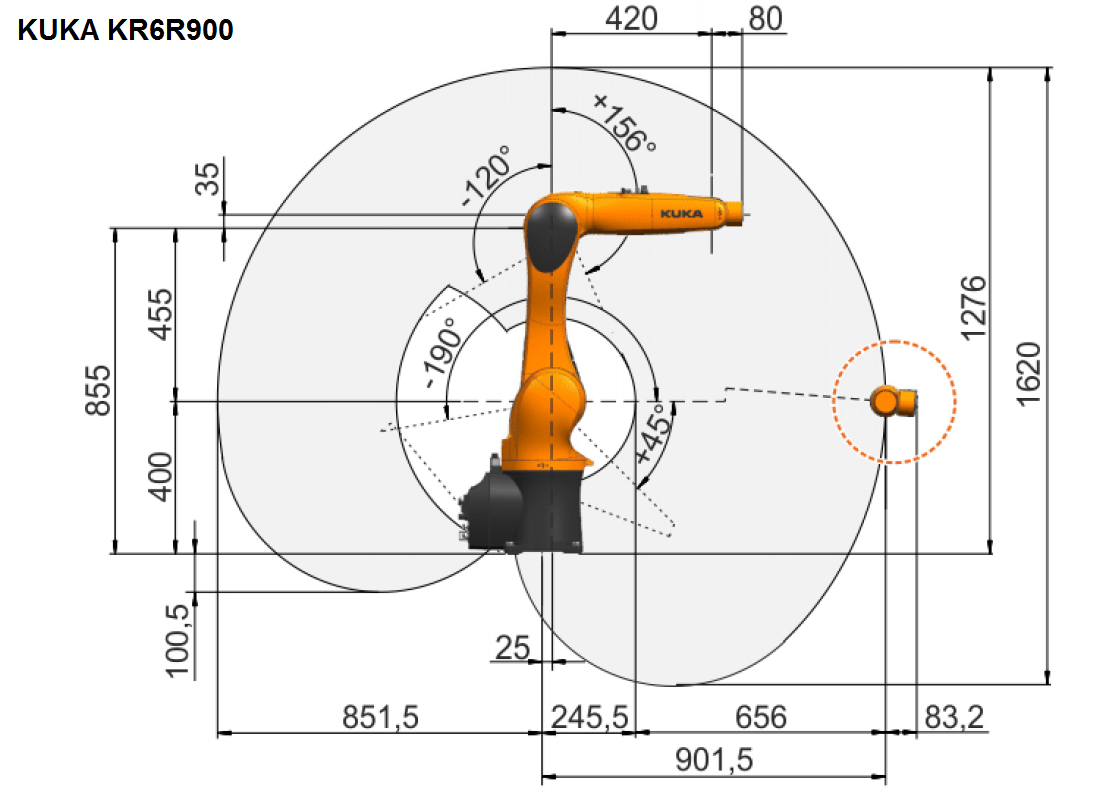}
\hfill
\includegraphics[width=0.42\textwidth]{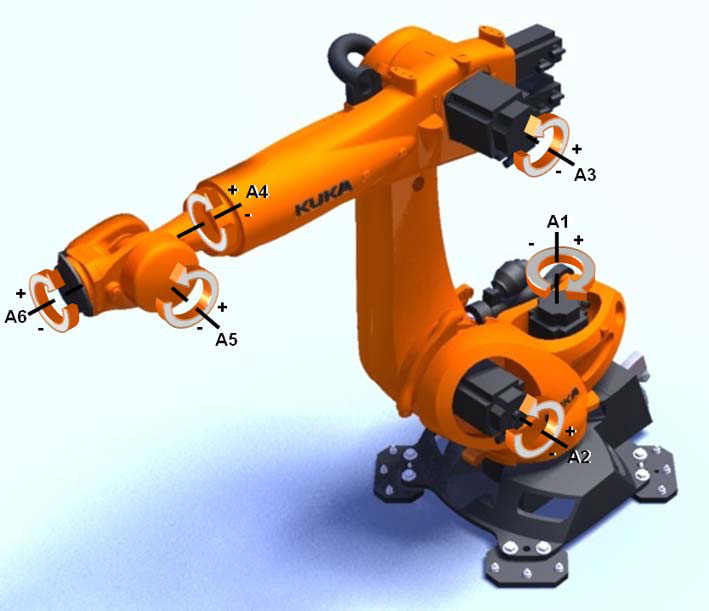}
\hfill

  \caption{KUKA KR6R900 dimensions and sense of axis rotation}
  \label{fig:KR6R900construction}
\end{figure}

We consider a 6R robot with spherical wrist and 
kinematic structure typical of many industrial robots. The robot is modelled in the 
Denavit-Hartenberg convention of the Robotics Toolbox described in \cite{Corke.2011}: The frame relating the coordinate systems of axes $i$ and $i+1$ is 
\[
	\hframe {i}{i+1}(\vartheta_i, d_i, a_i, \alpha_i,\varphi_i) = 
	 \rotz{\vartheta_i+\varphi_i}\transz{d_i}\transx{a_i}\rotx{\alpha_i} 
\]
with the usual abbreviations for translations and rotations. In addition to the familiar 
$\vartheta$, $d$, $a$, $\alpha$ the offset $\varphi$ gives an additional degree of freedom to assign
the axis zero positions as the mechanical engineers prefer.

We use a KUKA KR6R900 robot \cite{KSS83SI} industrial robot with 6 kg payload
and 900 mm reach. 
A construction drawing with the robot in its home 
position $(0, -\frac \pi 2, \frac \pi 2, 0, 0, 0)$ as well as the axis sense of rotation are
shown in Figure \ref{fig:KR6R900construction}, 
parameters 
and axis limits in Table \ref{fig:DHparametersOriginal} derived from \cite{AgilusSpecification.2013}.
Note that axis 1 is pointing downward in the
manufacturers definition, so the kinematic chain starts with an additional $\rotx \pi$.

This robot is mapped to a virtual robot with an additional prismatic joint with variable $v=d_4$ 
between the original axes 3 and 4, shown in the same Table \ref{fig:DHparametersOriginal} with 
tildes over the $\tilde \taxis$ variables. 
We distinguish between the variables of the two robots with indices $i$ and $j$,  
$\taxis = (\taxis_1, \ldots, \taxis_6)$ and 
$\axis = (\tilde\taxis_1,\tilde\taxis_2,\tilde\taxis_3,v,\tilde\taxis_4,\tilde\taxis_5, \tilde\taxis_6)$
respectively.
Ignoring axis limits and singularites, the backward transform gives up to 8 discrete solutions indexed by 
$\status\in\Status = \mset{0, 1, \ldots, 7}$. 
The virtual robot axis ranges are $\tilde\taxis_j \in (-\pi, \pi]$ for all rotational axes, and 
$v\in(-\infty, +\infty)$ for the virtual axis. 

\begin{table}
\hfill
 \begin{tabular}{c|c||c|c|c|c|c|c|c|c|c}
       \ $i$ \ & \ $j$ \  & \ $\taxis_i$ \ & \ $\axis_j$ \ & $d_i=d_j$  &  $a_i=a_j$     & $\alpha_i=\alpha_j$ & $\varphi_i=\varphi_j$ &  \ type\ & \hspace{0.3em} $\taxis_{\min,i}$ \hspace{0.3em} & \hspace{0.3em} $\taxis_{\max,i}$ \hspace{0.3em}
       \\[0.3ex] 
	\hline & & & & & & & & & \\[-1.9ex] 
         1   & 1 & $\taxis_1$ & $\tilde\taxis_1$ &   -400     &   25     &  $\frac \pi 2$ &      0         & R
         	& $-170^o$ & $170^o$\\
         2   & 2 & $\taxis_2$ & $\tilde\taxis_2$ &     0      &  455     &      0         &      0         & R
         	& $-190^o$ & $45^o$\\
         3   & 3 & $\taxis_3$ & $\tilde\taxis_3$ &     0      &   35     &  $\frac \pi 2$ & $-\frac \pi 2$ & R
         	& $-120^o$ & $156^o$\\
             & 4 &    0       &      $v$         &     0      &    0     &      0         &      0         & P 
         	& & \\
         4   & 5 & $\taxis_4$ & $\tilde\taxis_4$ &   -420     &    0     & $-\frac \pi 2$ &      0         & R
         	& $-185^o$ & $185^o$\\
         5   & 6 & $\taxis_5$ & $\tilde\taxis_5$ &     0      &    0     &  $\frac \pi 2$ &      0         & R
         	& $-120^o$ & $120^o$ \\
         6   & 7 & $\taxis_6$ & $\tilde\taxis_6$ &   -80      &    0     &    $\pi$       &      0         & R
            & $-350^o$ & $350^o$
  \end{tabular}
  \hfill\strut
  \caption{Denavit-Hartenberg parameters and axis limits in degrees}
  \label{fig:DHparametersOriginal}
\end{table}

We denote by
$\forward: (-\pi, \pi]^6  \to \frames \times \Status$, $\taxis \mapsto (P, \status)$
the forward transform of the original robot,
$\frames\subset \R^{4\times 4}$ denoting the set of frames, 
\ie $\SE 3$ in matrix representation. \forward  returns the tool frame in world coordinates and 
also the configuration, hereby making $\forward$ 
injective with an inverse $\backward$ on the set of nonsingular axis positions. The configuration
is a function $\status = \fStatus(\taxis)$; the formulas are not explicitly presented here: 
Bit 0 of $\status$ 
is set if the wrist centre point is behind axis 1. Bit 1 is set if the wrist centre point is 
below the line connecting axes 2 and 3. Bit 2 is set if axis 5 is directed upward. 
These bits are not simply signs of axis angles because there are offsets between axis 1 and 2 as well as
between axis 3 and 4, so the wrist centre point is not above axis 1 or on the line from axis 2 to 3 
if the robot is in an upright position. 
We overload the meaning of $\forward$ and also write $\forward(\axis)$ for the 7 axis virtual robot.
Analogously, $\backward$ denotes the backward transform in world coordinates
\begin{eqnarray*}
 	\backward: \frames \times \Status & \to & \prod_{i=1}^6 [\taxis_{\min,i}, \taxis_{\max,i}] \unionset \mset{\infty}
\\
	(P,\status) & \mapsto& \taxis \qquad \text{ with } \forward(\taxis) = P \text{ and } \fStatus(\taxis) = \status
\end{eqnarray*}
The special value $\infty$ in the range is used to signal unreachable points when 
the wrist centre point is outside the working space, or a solution with the given configuration exists, 
but outside the axis range.

The backward transform of the virtual robot is identical
to the standard 6R backward transform, if the TCP frame is reachable for the 6R robot. 
If not, the virtual axis is elongated to the minimum length (in absolute value) necesssary for the wrist centre
point to make the target point reachable; such a solution always exists for our robot class. This gives a 
well-defined function (for details see \cite{Wei.2019}, including a smoothing operation at the workspace boundary).
The key idea to replace the error signal $\infty$ in the non-solvable case by the quantity $\abs{v}$ of the virtual robot to measure non-solvability of the original backward transform. 

Our implementation can handle the class of 6R robots with the $\alpha_i$ values exactly as given
in Table \ref{fig:DHparametersOriginal}, 
and all the $d_i$, $a_i$ possibly non-zero as in Table \ref{fig:DHparametersOriginal}. 
The virtual axis could also be inserted between axes 2 and 3 with no effect on the rest of the analysis.
Other locations would not work -- \eg before axis 1 or after axis 6 --  or change the kinematic 
structure, \eg between axes of the central wrist. 
The approach carries over to any robot class with analytic solution and a similar virtual axis extension.

\section{Optimization with Configuration Selection}

Given a workpiece frame $F$ it is easy to check whether all process frames $P^k$ are reachable: 
Evaluate the backward transform for each $P^k$ with all possible configurations $\status$, 
$\taxis^k_{\status} = \backward(\hframe \world F\, {P^k}, \status)$, where $\hframe \world F\, {P^k}$ denotes
the change of coordinates from $F$ to $\world$. If at least one solution exists for all $P^k$, 
choose any of these solutions. If however at least one $k$ exists such that 
$\backward(\hframe \world F\, {P^k}, \status)=\infty$ signals an unreachable frame for all configurations
$\status$, we have no mathematical clue how to change $F$. This clue now comes from the virtual axis value $v$. 

We explain the algorithm in three steps: First we reformulate the reachability check for a single 
$P$ as a minimization problem, without axis restrictions. Then axis ranges are included. Finally 
we consider several $P^k$ and variations. 

\paragraph{minimin problem}
Given a single TCP frame $P$, reachability for workpiece location $F$ 
can be expressed for the original robot as follows: 
There exists $\status\in\Status$ such that $\backward(\hframe \world F P, c) \neq \infty$. For the 
virtual robot this is equivalent to: There exists $\status\in\Status$ such that
 $v(\axis)=0$ where $\axis = \backward(\hvec \world F\, P, c)$, and $v(\axis)$
is the projection onto the $v$-component. 
Abbreviating $f_\status(P) = (v(\backward(\hframe W F\, P, c))^2$, this can be rephrased as as a minimization
problem as in \cite{Wei.2019}: If $\min_{\status\in\Status} f_\status(P)$ has minimum value 0, we have
found a reachable solution. 
The square gives differentiability but can  
be replaced by the absolute value or any other distance function. The nondifferentiability of the absolute
value can also be handled in the optimization problem, see Section 4.

However, we have to optimize over $F$ giving $\min_F \min_{\status\in\Status} f_\status(P)$
(Note that $\min_{\status\in\Status} \min_F f_\status(P)$ is not the problem we are interested in: 
this means different positions $F$ for each configuration). This 
is a minimin-optimization problem considered difficult in literature, because 
the minimum over a finite number of differentiable functions is not differentiable but only 
continuous. This seems to exclude optimizers using derivatives that wen want to employ for efficiency.

So we have to reformulate the problem. For unconstrained minimax problems like the 
minimization of the $\infty$-norm there exist standard transformations to smooth, even linear,
but constrained formulations, see textbooks like \cite{Luenberger.2016}.  
For minimin, \cite[Exercise 12.6]{Nocedal.2006} leaves the problem unanswered, hinting that no differentiable 
formulation exists.
\cite[Chapter 8]{Eiselt.2007}  
suggests a sequence of linear problems, one for each $\status$, but only for a linear original problem. 
A key developer of the Gurobi optimization package 
suggests mixed nonlinear-integer programming \cite{GregGlockner.2016}, with a bit $b_\status \in\mset{0,1}$
for each $\status$ encoding with $b_\status=0$ whether $f_c$ attains the minimum, and an additional constraint 
$\sum_{\status\in\Status} b_\status = \abs{\Status}-1$ enforcing exactly one minimizing function. This excludes
a variable number of functions $f_\status$ reaching the minimum simultaneaously, and requires specialized software.

We propose a different transformation to a smooth constrained problem -- as smooth as the $f_\status$ -- 
with a convex combination which seems to be
new. We state the lemma with $x$ as the standard optimization variable, in our application
any parametrization of the workpiece frame $F=F(x)$ with $F = F(x,y,z,\alpha,\beta,\gamma)$.

\def\lvar{x}  
\def\lvarstar{x^\star}
 
\begin{lemma}
Consider $f_\status: \R^n\to \R$ for $\status \in \Status$, $\Status$ a finite set. 
Consider the unconstrained minimization problem
\begin{equation}
 \label{prob:miniminUnconstr}
\min_{\lvar\in\R^n} \min_{\status\in\Status} f_\status(\lvar)
\end{equation}
and the constrained problem
\begin{align}
\label{prob:miniminConstr}
	\min_{\lvar\in\R^n,w\in\R^{\abs{\Status}}} & \sum_{\status\in\Status} w_\status f_\status(\lvar)  
	\\
	\st & \sum_{\status\in\Status} w_\status = 1, \qquad
	0 \leq w_\status\leq 1 \qquad \text{ for all } \status\in\Status
	\nonumber
\end{align}
Then the problems are equivalent: \refe{prob:miniminUnconstr} is unbounded iff  \refe{prob:miniminConstr} is.
If the problems are bounded, then the minimum and infimum values are the same, $\mstar=f_\statusstar(\lvarstar)$ 
 for some $\lvarstar, \statusstar$ if the minimum is attained. 
\end{lemma}
\begin{proof}
The constraint $\sum_{\status\in\Status} w_\status = 1$ enforces that 
at least one function attains its minimum, if a minimum exists.

Assume a finite minimum 
$\mstar = f_\sstar(\lvarstar)$ for some $\sstar\in\Status$ for \refe{prob:miniminUnconstr}. 
We show that $\lvarstar$ and 
$\wstar$ with $\wstar_\sstar=1$ and $\wstar_\status=0$ for all $\status\neq \sstar$ is optimal for
 \refe{prob:miniminConstr}.
Clearly $\xstar, \wstar$ are admissible with objective function value $\mstar$. 
Choose any $x, w$ admissible.
The $w_\status$ are nonnegative and sum to 1, so we get
\[
	m := \sum_{\status\in\Status}  w_\status f_\status( x) 
	\geq 
	\sum_{\status\in\Status}  w_\status \left(\min_{\status\in\Status} f_\status( x) \right)
	= 
	\sum_{\status\in\Status}  w_\status  \mstar
	= \mstar
\]  
Therefore $\lvar,w$ cannot attain a smaller objective value. 

In the other direction, assume a finite minimum $\mstar$ for \refe{prob:miniminConstr} at $\lvarstar$, $\wstar$.
If $\wstar_{\status_r}>0$ for several $\status_r$, then all corresponding $f_{\status_r}(\xstar)$ must take the same
value - otherwise we could choose the minimum over $r$, increase the corresponding weight $\wstar_\sstar$, 
adjust the convex combination and so reduce the objective value. But with identical function values 
we can move the weights in $\wstar$ to a singleton with $\wstar_\sstar = 1$, $\wstar_\status=0$ otherwise.
As before $\xstar$ is the minimum point for \refe{prob:miniminUnconstr}.
The arguments for infimum and unboundedness are similar.
\end{proof}

\paragraph{Axis range restrictions}
Now we can include axis constraints into the smooth reformulation of the minimin problem. 
We sill consider
a single process frame $P\in\frames$. We identify $f_\status(x) = v_\status^2$ where 
$v_\status$ is the virtual axis value from
$(\tilde\taxis_{\status,1},\tilde\taxis_{\status,2},\tilde\taxis_{\status,3},v_\status,
\tilde\taxis_{\status,4},\tilde\taxis_{\status,5}, \tilde\taxis_{\status,6}) = 
\backward(\hframe \world {F(x)}\, P,\status)$, assuming no constraints on the axes, and now with $F$ depending
on $x$. 
If $\min_{\status\in \Status} f_\status(x) = 0$, then $\hframe \world {F(x)} \, P$ is reachable
for at least one configuration.

However the axis constraints are special: 
We need $\taxis_{\min,i} \leq \tilde\taxis_{\status,i} \leq \taxis_{\max,i}$ only for those configurations
with $v_\status = 0$: If $P$ needs an elongated virtual axis, then $P$ is unreachable anyway for
the original robot anyway, no matter whether the original axes restrictions are fulfilled. 
We model this with a slack variable $m_\status$ for the violation of the axis restrictions
for configuration $\status$, which is also included in the objective function (we suppress the
dependence of $\tilde\taxis$ on $x$):
\begin{align}
\label{prob:miniminP}
	\min_{\lvar\in\R^6,w\in\R^{\abs{\Status}}} & 
	\sum_{\status\in\Status} w_\status (f_\status(x) + m_\status)
	\nonumber \\
	\st & \sum_{\status\in\Status} w_\status = 1,  \quad 0 \leq w_\status\leq 1 
		& & \qquad \text{ for all }  \status\in\Status 
	\nonumber\\
	& \taxis_{\min,i} - \tilde\taxis_{\status,i}  \leq m_\status
	   & & \qquad \text{ for all }  \status \in \Status, i\in\mset{1, \ldots, 6}
	\\
	& \tilde\taxis_{\status,i} -  \taxis_{\max,i} \leq m_\status 
	   &  & \qquad \text{ for all }  \status \in \Status, i\in\mset{1, \ldots, 6}
	\nonumber \\
	&m_\status  \geq 0 & & \qquad \text{ for all }  \status \in \Status \nonumber
\end{align}
If all axis restrictions can be met with $m_\status=0$ we have found a solution for the original robot.
With the same argument as in the minimin lemma an objective value 0 signals reachability for at $P$
with at least one configurations. The measures $f_\status$ for violation of workspace and $w_\status$
for violation of axis range are nonnegative, yielding 0 for a reachable pose for the original  robot. 
Summing this up, we have replaced the test whether $\hframe \world {F(x)} \, P$ is 
reachable by a optimization problem to find $x$ and $F(x)$ respectively. 
Individual slack variables $m_{\min,\status,i}$ and $m_{\max,\status,i}$ for all axis constraints 
would also do but increase the number of variables.

If several $f_\sstar$ reach the same minimum value, then the proof of the Lemma shows that 
any configuration $\status$ with $w_\status > 0$ can be chosen in the application, this is still a degree 
of freedom. 

\paragraph{Formulation with frame list and variations}

For the original problem with several frames $P^1, \ldots, P^K$ we simply add indices $k$ to the
variables $w^k_\status$, $m^k_\status$, $v^k_\status$, 
$\tilde \taxis^k_{\status,i}$ of \refe{prob:miniminP} 
(but not $x$, which describes the single workpiece) and sum over all frames in the objective function 
$ \sum_{k=1}^K \sum_{\status\in\Status} w^k_\status ((v^k_\status)^2  + m^k_\status) $. 

Note that then we still only have one slack $m_\status^k$ 
variable for each configuration at each $P^k$, measuring axis range violations. This is the 
correct formulation for path point-to-point processes like handling. For path processes with blending 
contours a change of configuration must not occur on a path enclosed by two stop points 
because this would mean that a singularity is
crossed. However this requirement can be modelled easily: We use one common slack variable 
$\bar m_\status$ for all $P^k$ on the same path segment, but different variables
for different segments.

User constraints on $F$, like a rectancle of possible positions on a table, can be expressed in $x$ with easily.

\section{Numerical Results}

We have implemented the optimization problem in MATLAB with the SQP solver of
the Optimization Toolbox. For problems around $K=30$ a solution is found in typically 
less than 5 minutes on a standard laptop with i7 processor, with parallel evaluation
of the computations for the $P^k$. Figure \ref{fig:Optimization} shows a typical initial and optimized
positions. The virtual axis is plotted in red, so the robot is reaching to an unreachable point
in the initial setup. In the optimized solution, $v=0$ makes the virtual axis invisible.

We have also compared minimizing the absolute value 
$\abs{v}$ instead of $v^2$ using the standard transformation
$v = v_+ - v_-$, with $v_+, v_-\geq 0$ from linear programming, giving $\abs v = v_+ + v_-$. 
Numerical experiments showed better overall performance because $v^2$ has small
derivatives near the optimium, slowing down progress. The additional variables 
$v^{k,\status}_+$, $v^{k,\status}_-$ for all $k$ and $\status$ increase computation time for 
derivatives but reduced the number of iterations significantly. 
Due to the nonlinearity of the problem, the solver sometimes got stuck in an
non-reachable workpiece frame.
\begin{figure}
\hfill
\includegraphics[width=0.49\textwidth]{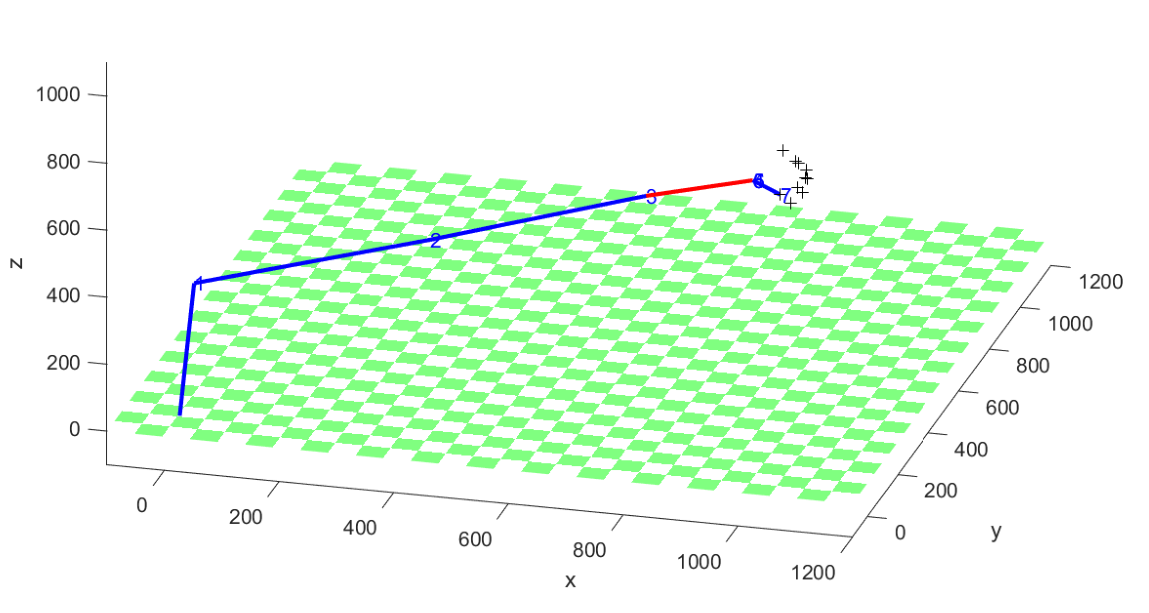}
\hfill 
\includegraphics[width=0.49\textwidth]{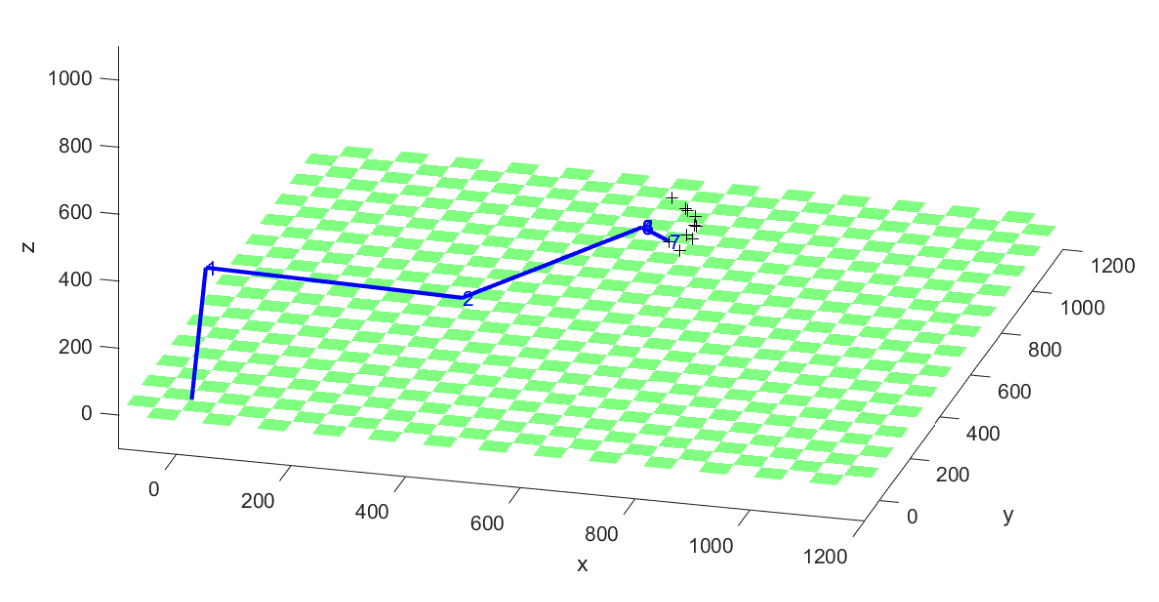}
\hfill
  \caption{Initial and optimized position}
  \label{fig:Optimization}
\end{figure}


\section{Conclusion and Outlook}
We have presented a general purpose algorithms that places a workpiece in the reachable workspace of a 
6R industrial robot, capaple of selecting configurations for all process points. 
A natural next step is a time optimal path 
through one of the possible configurations of each process point in 
a travelling salesman approach \cite{Laporte.1983}.
We found no academic software that
offers configuration programming at all. In the author's opinion, this  would be a very useful functionality.


\bibliographystyle{spmpsci}
\bibliography{Literature}

\end{document}